%% file: acl_latex.tex
\documentclass[11pt]{article}
\usepackage{multirow}
\usepackage{placeins}
\usepackage[preprint]{acl}
\usepackage{amsmath}
\usepackage{amssymb}
\usepackage{algorithm}
\usepackage{algpseudocode}
\usepackage{times}
\usepackage{latexsym}
\usepackage{booktabs}
\usepackage[T1]{fontenc}
\usepackage{fontawesome5}

\usepackage{enumitem}
\usepackage{amsthm}
\newtheorem{assumption}{Assumption}
\newtheorem{proposition}{Proposition}
\newtheorem{corollary}{Corollary}
\newtheorem{definition}{Definition}
\newtheorem{remark}{Remark}
\newtheorem*{proposition*}{Proposition}\newtheorem*{assumption*}{Assumption}

\usepackage[utf8]{inputenc}

\usepackage{microtype}

\usepackage{inconsolata}

\usepackage{graphicx}

\title{AMBER: Multi-View Adaptive Budget Allocation for Listwise Vision-Language Reranking}

\author{
  \textbf{Wenteng Chen\textsuperscript{1,*}},
  \textbf{Jiachen Zhu\textsuperscript{1,*}},
  \textbf{Rong Shan\textsuperscript{1,3}},
  \textbf{Tianyi Xu\textsuperscript{1}},
  \textbf{Yuxiang Chen\textsuperscript{4}},
  \\
  \textbf{Congmin Zheng\textsuperscript{1}},
  \textbf{Teng Wang\textsuperscript{2}},
  \textbf{Junjie Wu\textsuperscript{2}},
  \textbf{Weiwen Liu\textsuperscript{1}},
  \textbf{Changwang Zhang\textsuperscript{2}},
  \\
  \textbf{Weinan Zhang\textsuperscript{1}},
  \textbf{Jun Wang\textsuperscript{2,\textdagger}},
  \textbf{Jianghao Lin\textsuperscript{1,\textdagger}}
  \\
  \\
  \textsuperscript{1}Shanghai Jiao Tong University,
  \textsuperscript{2}OPPO,
  \textsuperscript{3}Shanghai Innovation Institute,
  \textsuperscript{4}University College London
  \\
  \small{
    \textsuperscript{*}Equal contribution.
    \quad
    \textsuperscript{\textdagger}Corresponding authors.
  }
  \\
  \small{
    \texttt{\{cwt-03,gebro13,linjianghao\}@sjtu.edu.cn}
  }
}

\begin{document}
\maketitle
\begin{abstract}

Vision-language models (VLMs) are powerful listwise rerankers for multimodal retrieval, but high inference costs restrict them to evaluating small local candidate views. Existing multi-call strategies rely on fixed schedules, wasting expensive VLM calls on uninformative candidate pairs and easy queries. To address this, we propose \textbf{A}daptive \textbf{M}ulti-view \textbf{B}udgeted \textbf{E}lo \textbf{R}eranking (\textbf{AMBER}), an online, budgeted multi-view reranking framework that dynamically optimizes global resource allocation. AMBER treats fragmented listwise VLM outputs as local tournaments, using continuous Elo updates to maintain a lightweight global ranking state. Building on this, it allocates computation at two levels: dynamically constructing candidate views with high score ambiguity, and scheduling queries to maximize expected information gain. We show that each Elo update corresponds to a stochastic gradient ascent step on the Bradley--Terry log-likelihood, and provide a submodular information-theoretic motivation for the query-level allocation strategy. Experiments on CIRR, CIRCO, and PhotoBench demonstrate that AMBER achieves the strongest overall performance among the compared multi-call VLM reranking methods under comparable VLM-call budgets, while remaining effective in lower-budget settings. Our code is publicly available at
\href{https://github.com/wnlfc/AMBER}{\faGithub\ \texttt{AMBER}}.
\end{abstract}

\input{text/introduction}

\input{text/preliminary}

\input{text/methodology}

\input{text/experiments}

\input{text/related_works}

\input{text/conclusion}

\input{text/limitation}

\bibliography{custom}

\appendix

\input{text/Appendix}

\end{document}

%% file: text/introduction.tex
\section{Introduction}

Multimodal retrieval typically follows a two-stage paradigm, where a strong reranker refines candidates recalled by an embedding model. Vision-language models (VLMs) excel at this second stage via listwise reranking, which compares multiple images in a shared context to capture fine-grained visual-textual alignment \citep{cai2026visionmeetstextslistwise, qiao2026llmbasedlistwisererankingeffect}. However, this approach introduces a severe resource bottleneck. Due to token costs, context length limits, and inference latency, a single VLM call can only observe a small \emph{local view} of candidates. Reliable reranking thus requires aggregating evidence across multiple VLM calls, which creates a critical \textbf{budget allocation problem}: how should a system optimally distribute a limited computation budget across thousands of queries and candidates?

Existing multi-call strategies fail to utilize this budget effectively. Approaches like sliding-window or tournament-style traversals rely on fixed schedules \citep{chen2025tourrankutilizinglargelanguage}. They apply the same reranking effort uniformly, failing to distinguish which candidate comparisons are actually informative or which queries are fundamentally more difficult. While recent adaptive methods \citep{yoon2025acurankuncertaintyawareadaptivecomputation} can stop early for confident queries, they operate in isolation and cannot dynamically redistribute the remaining budget across a global pool of competing queries. Consequently, existing systems waste expensive VLM calls on easy queries or uninformative candidate pairs \citep{rashid2024ecorankbudgetconstrainedtextreranking, liu2026stopoverthinkingunlockingefficient}.

To address this,  we propose \textbf{A}daptive \textbf{M}ulti-view \textbf{B}udgeted \textbf{E}lo \textbf{R}eranking (\textbf{AMBER}), an online, budgeted multi-view reranking framework that jointly optimizes evidence aggregation and dynamic resource allocation. To aggregate fragmented local views, AMBER treats each listwise VLM output as a local tournament, converting the outcomes into online Elo score updates. This yields a lightweight global ranking state that seamlessly integrates first-stage embedding priors with VLM evidence, without requiring an explicit pairwise preference matrix.

Building on this dynamic state, AMBER strategically allocates computation at two levels. At the \emph{candidate level}, it constructs the next local view by prioritizing candidates that are near the top of the ranking, infrequently observed, and have ambiguous scores compared to their neighbors. At the \emph{query level}, AMBER evaluates the expected information gain of these proposed views, scheduling VLM calls exclusively for queries where the comparison is highly uncertain and consequential for head-ranking precision.

Crucially, AMBER's design is heavily grounded in theory. We demonstrate that its Elo updates can be interpreted as online stochastic optimization of the Bradley--Terry model \citep{robbins1951stochastic}. Furthermore, an idealized query-level information-gain objective is a monotone submodular function, yielding a $(1-1/e)$ approximation guarantee for greedy allocation under this objective \citep{nemhauser1978analysis}. Extensive experiments on CIRR \citep{cirr}, CIRCO \citep{circo}, and PhotoBench \citep{xu2026photobench} show that AMBER achieves
strong overall performance among the compared multi-call VLM
reranking methods under comparable budgets and remains effective
in low-budget settings.

Our main contributions are as follows:
\begin{itemize}[leftmargin=5pt]
    \vspace{-6pt}
    \item We are the first to formulate VLM-based listwise reranking as a \textbf{budgeted multi-view evidence aggregation problem}, moving beyond fixed schedules to dynamic global resource allocation.
    \vspace{-8pt}
    \item We propose \textbf{AMBER}, an online framework that maintains a lightweight global ranking state by converting local listwise VLM outputs into continuous Elo-style updates.
    \vspace{-8pt}
    \item We design a \textbf{two-level allocation strategy} that jointly selects high-uncertainty candidate views and schedules queries based on maximum expected information gain.
    \vspace{-8pt}
    \item We provide a rigorous \textbf{theoretical foundation}, connecting AMBER's Elo updates to online maximum-likelihood optimization and establishing submodular approximation guarantees for an idealized information-gain objective \citep{nemhauser1978analysis}, complemented by strong results on three multimodal retrieval benchmarks.
    \vspace{-5pt}
    
\end{itemize}

%% file: text/preliminary.tex
\section{Preliminary}
We study second-stage vision-language model (VLM) reranking for
multimodal retrieval.
Let $\mathcal{Q}=\{q_{1},q_{2},\ldots,q_{m}\}$ denote the full query
set of size $m$.
For each query $q_{j}\in\mathcal{Q}$, a first-stage embedding retriever
returns a candidate pool ordered by embedding similarity:
\begin{equation}
  \mathcal{C}_{q_{j}} = \{c_{1}, c_{2}, \ldots, c_{N}\}.
  \label{eq:candidate_pool}
\end{equation}
The ultimate goal is to produce a refined ranking $\hat{\sigma}_{q_{j}}$
that places the most relevant candidates at the top.

A listwise VLM reranker takes $q_{j}$ and a small candidate subset as
input to output a local ranking.
At any reranking round $t$, the VLM can only observe a limited
\emph{local view}
\begin{equation}
  V_{q_{j}}^{(t)} \subset \mathcal{C}_{q_{j}},
  \qquad
  \bigl|V_{q_{j}}^{(t)}\bigr| = M, \quad M \ll N.
  \label{eq:local_view}
\end{equation}
Given the substantial computational cost per VLM call, we assume a
fixed global call budget
\begin{equation}
  B = b \cdot m,
  \label{eq:budget}
\end{equation}
where $b$ is the average number of VLM calls permitted per query.

AMBER maintains a query-specific state for each $q_{j}$:
\begin{equation}
  \mathcal{S}_{q_{j}}
  = \Bigl(
      \{s_{ij}\}_{i=1}^{N},\;
      \{m_{ij}\}_{i=1}^{N},\;
      \mathcal{O}_{q_{j}},\;
      \{r_{ij}\}_{i=1}^{N},\;
      U_{q_{j}}
    \Bigr),
  \label{eq:state}
\end{equation}
where $s_{ij}$ is the current Elo score of candidate $c_{i}$ for query
$q_{j}$, $m_{ij}$ is its VLM observation count, $\mathcal{O}_{q_{j}}$
is the set of observed candidates, $r_{ij}$ is the global rank of
$c_{i}$ induced by the current Elo state for $q_{j}$, and $U_{q_{j}}$
is the query-level scheduling utility score.

%% file: text/methodology.tex
\section{Methodology}

\begin{figure*}[t]
  \centering
  \vspace{-15pt}
  \includegraphics[width=1.0\textwidth]{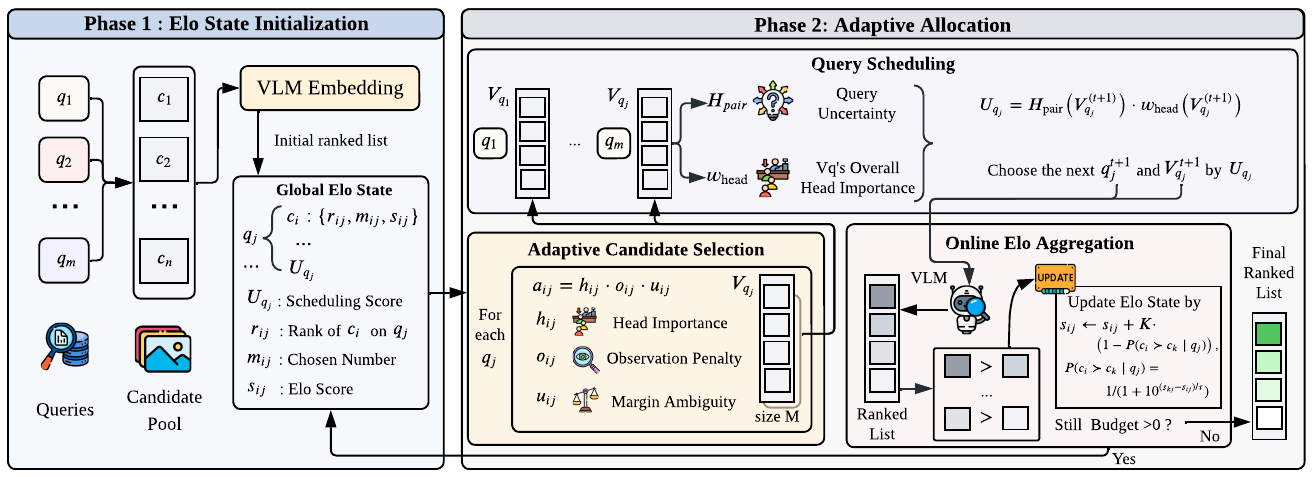}
  \vspace{-20pt}
  \caption{The overall architecture of AMBER.}
  \vspace{-13pt}
  \label{fig: Introduction}
\end{figure*}

To address the aforementioned challenges under a strict budget
constraint, we propose \textbf{AMBER}, a dynamic framework that
jointly optimises query allocation, candidate selection, and
local-to-global ranking aggregation.
We elaborate on the three core mechanisms below.

\subsection{Online Cross-View Elo Aggregation}

Since individual VLM calls observe different candidate subsets, the
aggregation of local rankings must be incremental and cross-view
compatible.
AMBER conceptualises this process as an online rating mechanism.
Given a local VLM ranking $\pi_{q_{j}}^{(t)}$ over view
$V_{q_{j}}^{(t)}$, we decompose the local ranking into pairwise outcomes.
If candidate $c_{i}$ is ranked above $c_{k}$ in $\pi_{q_{j}}^{(t)}$,
we record a ``win'':
\begin{equation}
  c_{i} \succ c_{k}.
  \label{eq:win}
\end{equation}
For each outcome $c_{i}\succ c_{k}$, the expected winning probability
based on the current state $\mathcal{S}_{q_{j}}$ is computed as:
\begin{equation}
  P(c_{i}\succ c_{k}\mid q_{j})
  = \frac{1}{1+10^{(s_{kj}-s_{ij})/\tau}},
  \label{eq:expected_win}
\end{equation}
where $\tau$ is the Elo scaling factor.
The candidates' scores are then updated simultaneously:
\begin{align}
  s_{ij} &\leftarrow s_{ij} + K\!\left(1-P(c_{i}\succ c_{k}\mid q_{j})\right),
  \label{eq:elo_update_i}\\[4pt]
  s_{kj} &\leftarrow s_{kj} - K\!\left(1-P(c_{i}\succ c_{k}\mid q_{j})\right),
  \label{eq:elo_update_k}
\end{align}
where $K$ dictates the update step size.

Initial Elo scores $\{s_{ij}\}$ are assigned using the first-stage
embedding ranking, acting as a prior. Specifically, for the candidate
ranked $i$-th in the embedding-based candidate pool
$\mathcal{C}_{q_j}$, we initialise
\begin{equation}
  s_{ij}^{(0)}
  =
  s_{0}
  +
  R\cdot
  \frac{N-i}{N-1},
  \label{eq:elo_init}
\end{equation}
where $N=|\mathcal{C}_{q_j}|$, $s_{0}$ is the base rating, and $R$
controls the strength of the embedding prior. This linear
initialisation gives higher-ranked candidates from the first-stage
retriever larger initial Elo scores, while allowing subsequent VLM
observations to revise the ranking through online updates.

Following each update, observed candidates are re-sorted by their Elo
scores, while unobserved candidates are appended based on their
original embedding order.

\subsection{Adaptive Candidate Selection}

When query $q_{j}$ is selected for a VLM call, AMBER must construct
the most informative local view $V_{q_{j}}^{(t+1)}$.
We assign each candidate $c_{i}$ a sampling value $a_{ij}$ that
integrates three strategic factors.

\paragraph{1.\ Head Importance.}
We prioritise candidates near the top of the current ranking, aligning
with metrics such as Recall@$K$ and mAP@$K$:
\begin{equation}
  h_{ij} = \frac{1}{\log_{2}(r_{ij}+1)}.
  \label{eq:head_importance}
\end{equation}

\paragraph{2.\ Observation Penalty.}
To prevent redundant evaluations, we introduce diminishing returns for
frequently observed candidates:
\begin{equation}
  o_{ij} = \frac{1}{\sqrt{m_{ij}+1}}.
  \label{eq:obs_penalty}
\end{equation}

\paragraph{3.\ Local Margin Ambiguity.}
We favour candidates whose relative ordering is highly uncertain.
Let $\mathcal{N}_{\kappa}^{(j)}(c_{i})$ denote the set of candidates
within $\kappa$ rank positions of $c_{i}$ under query $q_{j}$.
The minimal local Elo gap of $c_{i}$ is
\begin{equation}
  g_{ij} = \min_{c_{l}\,\in\,\mathcal{N}_{\kappa}^{(j)}(c_{i})}
            \left|s_{ij} - s_{lj}\right|,
\end{equation}
yielding the margin ambiguity factor:
\begin{equation}
  u_{ij} = \frac{1}{1+g_{ij}/\tau_{c}},
  \label{eq:margin_ambiguity}
\end{equation}
where $\tau_{c}$ controls gap sensitivity.
Smaller margins denote ambiguous ordering, increasing the sampling
value.

The final candidate sampling value is the product of these three
factors:

\begin{align}
  a_{ij}
  &= h_{ij}\cdot o_{ij}\cdot u_{ij} \nonumber \\
  &= \frac{1}{\log_{2}(r_{ij}+1)}
    \cdot\frac{1}{\sqrt{m_{ij}+1}}
    \cdot\frac{1}{1+g_{ij}/\tau_{c}}.
  \label{eq:sampling_value}
\end{align}
The system then constructs the next view by selecting the candidates
with the highest values:
\begin{equation}
  V_{q_{j}}^{(t+1)} = \operatorname{TopM}_{c_{i}\,\in\,\mathcal{C}_{q_{j}}}\;a_{ij}.
  \label{eq:next_view}
\end{equation}

\subsection{Entropy-Based Query Scheduling}

At the query level, AMBER dynamically allocates the remaining global
budget based on the expected utility of the \emph{next} VLM call.
Rather than assessing query difficulty in isolation, we condition the
scheduling score $U_{q_{j}}$ on the specific view $V_{q_{j}}^{(t+1)}$
that would be evaluated.

For each eligible query $q_{j}$, we simulate its next view
$V_{q_{j}}^{(t+1)}$ and estimate the pairwise uncertainty within it.
For every pair $(c_{i},c_{k})$ in this view, let
$p_{ik}^{(j)} := P(c_{i}\succ c_{k}\mid q_{j})$ denote their Elo
winning probability.
The normalised binary entropy derived from $p_{ik}^{(j)}$ is:
\begin{equation}
  \begin{aligned}
    H\!\left(p_{ik}^{(j)}\right) = {} & -\frac{1}{\log 2} \Big[ p_{ik}^{(j)}\log p_{ik}^{(j)} \\
    & + \left(1-p_{ik}^{(j)}\right)\log\!\left(1-p_{ik}^{(j)}\right) \Big].
  \end{aligned}
  \label{eq:pair_entropy}
\end{equation}

The view-level pairwise entropy is the average uncertainty over all
comparisons in the view:
\begin{equation}
  \begin{aligned}
    H_{\mathrm{pair}}\!\left(V_{q_{j}}^{(t+1)}\right) = {} & \frac{2}{M(M-1)} \\
    & \times \sum_{\substack{c_{i},\,c_{k}\,\in\,V_{q_{j}}^{(t+1)}\\i<k}} H\!\left(p_{ik}^{(j)}\right).
  \end{aligned}
  \label{eq:view_entropy}
\end{equation}
To ensure the budget serves top-ranking improvements, we scale this
entropy by the view's average head importance:
\begin{equation}
  w_{\mathrm{head}}\!\left(V_{q_{j}}^{(t+1)}\right)
  = \left(
      \frac{1}{M}
      \sum_{c_{i}\,\in\,V_{q_{j}}^{(t+1)}}
      \frac{1}{\log_{2}(r_{ij}+1)}
    \right)^{\!1/2}.
  \label{eq:head_weight}
\end{equation}
The final query scheduling score is formulated as:
\begin{equation}
  U_{q_{j}}
  = H_{\mathrm{pair}}\!\left(V_{q_{j}}^{(t+1)}\right)
    \cdot w_{\mathrm{head}}\!\left(V_{q_{j}}^{(t+1)}\right).
  \label{eq:scheduling_score}
\end{equation}
This ensures the system strictly prioritises queries where the
designated VLM call will resolve the most critical and uncertain head
comparisons.

\subsection{Overall Algorithm}

The overall AMBER framework operates in two distinct phases,
summarised in Algorithm~\ref{alg:AMBER}.

\paragraph{Phase~1 (Initial Reranking).}
To establish baseline observations, every query $q_{j}\in\mathcal{Q}$
receives exactly one VLM call.
Local rankings are obtained and translated into initial Elo updates,
initialising $\mathcal{S}_{q_{j}}$ for all $j$.

\paragraph{Phase~2 (Adaptive Allocation).}
The framework enters an iterative budget allocation loop.
While the total spent budget $S$ is less than $B$, AMBER computes
the query scheduling score $U_{q_{j}}$ for all eligible queries.
The queries with the highest $U_{q_{j}}$ are granted a VLM call for
their respective $V_{q_{j}}^{(t+1)}$.
Elo states $\mathcal{S}_{q_{j}}$ are continuously updated until the
budget is exhausted, at which point the final hybrid rankings
$\hat{\sigma}_{q_{j}}$ are returned.

\begin{algorithm}[t!]
\caption{AMBER Reranking}
\label{alg:AMBER}
\begin{algorithmic}[1]
\Require Query set $\mathcal{Q}=\{q_{1},\ldots,q_{m}\}$,
         candidate pools $\{\mathcal{C}_{q_{j}}\}_{j=1}^{m}$,
         view size $M$, total budget $B$
\Ensure  Final rankings $\{\hat{\sigma}_{q_{j}}\}_{j=1}^{m}$
\For{each query $q_{j}\in\mathcal{Q}$}
  \State Initialise Elo scores $\{s_{ij}\}$ from embedding ranks
  \State Initialise observation counts $m_{ij}\leftarrow 0$,
         observed set $\mathcal{O}_{q_{j}}\leftarrow\emptyset$
  \State Initialise global ranks $\{r_{ij}\}$ by embedding order
\EndFor
\State \textit{Phase 1: Initial Reranking}
\For{each query $q_{j}\in\mathcal{Q}$}
  \State Select view $V_{q_{j}}^{(1)}$ by candidate sampling
  \State Obtain local ranking $\pi_{q_{j}}^{(1)}$ from the VLM
  \State Convert $\pi_{q_{j}}^{(1)}$ into pairwise outcomes;
         update $\{s_{ij}\}$, $\{m_{ij}\}$, $\mathcal{O}_{q_{j}}$,
         and $\{r_{ij}\}$
\EndFor
\State $S \leftarrow m$
\State \textit{Phase 2: Adaptive Allocation}
\While{$S < B$}
  \For{each eligible query $q_{j}$}
    \State Build next view $V_{q_{j}}^{(t+1)}$ via
           Eq.~\eqref{eq:next_view}
    \State Compute scheduling score
           $U_{q_{j}} = H_{\mathrm{pair}}(V_{q_{j}}^{(t+1)})
                        \cdot w_{\mathrm{head}}(V_{q_{j}}^{(t+1)})$
  \EndFor
  \State Select queries with the largest $U_{q_{j}}$
  \For{each selected query $q_{j}$}
    \State Rerank $V_{q_{j}}^{(t+1)}$ with the VLM;
           obtain $\pi_{q_{j}}^{(t+1)}$
    \State Update $\{s_{ij}\}$ from pairwise outcomes
    \State Update $\{m_{ij}\}$, $\mathcal{O}_{q_{j}}$,
           and $\{r_{ij}\}$
  \EndFor
  \State Update spent budget $S$
\EndWhile
\For{each query $q_{j}\in\mathcal{Q}$}
  \State Sort observed candidates by Elo scores
  \State Append unobserved candidates by embedding order
  \State Return final ranking $\hat{\sigma}_{q_{j}}$
\EndFor
\end{algorithmic}
\end{algorithm}

\subsection{Theoretical Analysis}
\label{sec:theory}

We provide theoretical grounding for AMBER's three core mechanisms
under a generative assumption on VLM outputs.
All proofs are deferred to Appendix~\ref{app:proofs}.

\begin{assumption}[Bradley--Terry Model]
\label{ass:bt}
The pairwise preferences produced by the VLM follow a Bradley--Terry
(BT) model: there exist true latent scores
$\mathbf{s}^{*}=\{s_{i}^{*}\}_{i=1}^{N}$ such that for any pair
$(c_{i},c_{k})$,
\begin{equation}
  P^{*}(c_{i}\succ c_{k})=\frac{1}{1+10^{(s_{k}^{*}-s_{i}^{*})/\tau}},
  \label{eq:bt}
\end{equation}
and VLM outputs for distinct pairs are conditionally independent given
$\mathbf{s}^{*}$.
\end{assumption}

\begin{proposition}[Elo Updates as Online MLE]
\label{prop:mle}
Under Assumption~\ref{ass:bt}, each Elo update
\emph{(Eqs.~\eqref{eq:elo_update_i}--\eqref{eq:elo_update_k})}
applied to a single observed outcome $c_{i}\succ c_{k}$ is equivalent
to a stochastic gradient ascent step on the BT log-likelihood with
effective learning rate $\eta=K\tau/\ln 10$.
\end{proposition}

\begin{corollary}[Convergence Rate]
\label{cor:convergence}
Under Assumption~\ref{ass:bt} and a weak $\ell_{2}$-regularized BT
objective, with decaying step size $K_{t}=K_{0}/\sqrt{t}$, the Elo
estimates satisfy
\begin{equation}
  \mathbb{E}\!\left[
    \ell\!\left(\mathbf{s}^{\mathrm{MLE}}\right)
    -\ell\!\left(\hat{\mathbf{s}}_{T}\right)
  \right]
  =O\!\left(\frac{1}{\sqrt{T}}\right),
\end{equation}
where $T$ is the total number of processed pairwise outcomes.
\end{corollary}

\begin{proposition}[Sample Complexity for Top-$K$ Recovery]
\label{prop:sample}
Under Assumption~\ref{ass:bt}, suppose the true scores satisfy
$s_{1}^{*}>s_{2}^{*}>\cdots>s_{N}^{*}$, and let
$\Delta_{K}=s_{K}^{*}-s_{K+1}^{*}$ denote the score gap at the
$K$-th boundary.  If every straddling pair $(c_{i},c_{k})$ with
$i\leq K<k$ is observed at least
\begin{equation}
  m \;\geq\; \frac{C\,\tau^{2}\ln(N/\delta)}{\Delta_{K}^{2}}
  \label{eq:sample_complexity}
\end{equation}
times for an absolute constant $C>0$, then with probability at least
$1-\delta$ the top-$K$ set induced by the Elo scores equals the true
top-$K$ set.
\end{proposition}

\noindent

\noindent
Proposition~\ref{prop:sample} motivates the three factors in
the candidate sampling value $a_{ij}$
(Eq.~\eqref{eq:sampling_value}):

\begin{itemize}[leftmargin=6pt]
  \item \textbf{Head importance} $h_{ij}$: the required observation
        count concerns only pairs straddling the top-$K$ boundary.
        Candidates near the head of the ranking are far more likely to
        be involved in such pairs, justifying
        $h_{ij}=1/\log_{2}(r_{ij}+1)$ as a proximity-to-boundary
        weight.

  \item \textbf{Observation penalty} $o_{ij}$: once a candidate has
        been observed $m_{ij}$ times, further comparisons yield
        diminishing reductions in estimation error (the
        $O(1/\sqrt{m})$ bound).
        The factor $o_{ij}=1/\sqrt{m_{ij}+1}$ implements exactly this
        diminishing-returns discount, preventing redundant allocation.

  \item \textbf{Margin ambiguity} $u_{ij}$:
        Eq.~\eqref{eq:sample_complexity} shows that the required
        observation count scales as $1/\Delta_{K}^{2}$.  Candidates
        whose current Elo gap $g_{ij}$ to their nearest neighbour is
        small have an estimated $\Delta_{K}\approx g_{ij}$ and
        therefore require the most additional observations.  The factor
        $u_{ij}=1/(1+g_{ij}/\tau_{c})$ assigns higher priority to
        precisely these ambiguous candidates.
\end{itemize}

\begin{definition}
\label{def:uncertainty_reduction}
For query $q_{j}$, let $\mathcal{A}_{q_{j}}$ denote the set of VLM
calls already made.  The \emph{uncertainty reduction} achieved by
$\mathcal{A}_{q_{j}}$ is
\begin{align}
  f_{q_{j}}(\mathcal{A}_{q_{j}}) = {} & H_{\mathrm{pair}}\!\left(V_{q_{j}}^{(t+1)}\mid\emptyset\right) \nonumber \\
  & - H_{\mathrm{pair}}\!\left(V_{q_{j}}^{(t+1)}\mid\mathcal{A}_{q_{j}}\right).
\end{align}
\end{definition}

\begin{proposition}[Submodularity of Uncertainty Reduction]
\label{prop:submodular}
Under Assumption~\ref{ass:bt}, $f_{q_{j}}$ is a non-negative,
monotone, submodular function of $\mathcal{A}_{q_{j}}$.
\end{proposition}

\begin{corollary}[Greedy Approximation Guarantee]
\label{cor:greedy}
Let $\mathcal{A}^{*}$ be the optimal allocation of the remaining budget
after Phase~1 that maximises $\sum_{j}f_{q_{j}}(\mathcal{A}_{q_{j}})$.
Under the idealized setting of Proposition~\ref{prop:submodular},
the greedy policy that selects the query with the largest exact marginal
gain at each step achieves
\begin{equation}
  \sum_{j}f_{q_{j}}(\hat{\mathcal{A}}_{q_{j}})
  \;\geq\;
  \left(1-\frac{1}{e}\right)\sum_{j}f_{q_{j}}(\mathcal{A}_{q_{j}}^{*}),
\end{equation}
where $\hat{\mathcal{A}}_{q_{j}}$ is the set of calls allocated by
this idealized greedy policy. The scheduling score $U_{q_{j}}$
(Eq.~\eqref{eq:scheduling_score}) used by AMBER serves as a tractable,
head-weighted proxy for the exact marginal gain
$f_{q_{j}}(\mathcal{A}_{q_{j}}\cup\{x\})-f_{q_{j}}(\mathcal{A}_{q_{j}})$.
\end{corollary}

\paragraph{Limitations.}
The results above rest on Assumption~\ref{ass:bt}. In practice, VLM
outputs may exhibit positional bias and context sensitivity, and
Proposition~\ref{prop:submodular} assumes conditional independence
across calls, whereas AMBER's Elo state couples successive
observations. The practical score $U_{q_j}$ is also a tractable proxy
for the theoretical marginal gain. These gaps motivate the empirical
evaluation in Section~\ref{sec:experiments}.

%% file: text/experiments.tex
\section{Experiments}
\label{sec:experiments}

\subsection{Experimental Setup}

\paragraph{Datasets.}
We evaluate our method on three multimodal retrieval benchmarks: CIRR~\cite{cirr},
CIRCO~\cite{circo}, and PhotoBench~\cite{xu2026photobench}. CIRR is a composed image retrieval benchmark
where each query consists of a reference image and a relative
modification text, and each query has one target image. This makes the
evaluation especially sensitive to the exact position of the target
image in the ranked list. CIRCO is a more challenging composed image
retrieval benchmark where each query may have multiple valid target
images. Therefore, in addition to Recall@$K$, mAP@$K$ is an important
metric for evaluating the quality of the ranked list. PhotoBench
evaluates retrieval over personal photo albums. We use the merged-album
setting and report results separately for Chinese and English queries.

\begin{table}[t]
\centering
\caption{Summary of dataset statistics, scales, and evaluation characteristics.}
\vspace{-10pt}
\label{tab:dataset_specs}
\renewcommand{\arraystretch}{1.2} 
\resizebox{\columnwidth}{!}{
\begin{tabular}{@{}lrrcl@{}}
\toprule
\textbf{Dataset} & \textbf{\# Images} & \textbf{\# Queries} & \textbf{Targets / Query} & \textbf{Language} \\ 
\midrule
\textbf{CIRR} & 36,000+ & 36,554 & Single & English \\ 
\addlinespace
\textbf{CIRCO} & 123,403 & 1,020 & Multiple (Avg. 4.53) & English \\ 
\addlinespace
\textbf{PhotoBench} & 3,582 & 1,188 & Intent-driven & EN / CN \\ 
\bottomrule
\vspace{-10pt}
\end{tabular}%
} 
\end{table}

\paragraph{Evaluation Metrics.}
For CIRR, we report Recall@1, Recall@2, Recall@5, Recall@10, and
Recall@50. For
CIRCO, we report Recall@5, Recall@10, Recall@25, Recall@50, and
mAP@5, mAP@10, mAP@25, and mAP@50. For PhotoBench, we report
Recall@$K$ and NDCG@$K$. Some PhotoBench queries have empty
ground-truth sets. We keep these queries in evaluation and assign zero
to all their metrics, rather than skipping them. All PhotoBench numbers
are multiplied by 100.

\paragraph{Implementation Details.}
All experiments use Qwen3-VL-Embedding-8B as the first-stage retriever
and Qwen3-VL-8B-Instruct as the VLM reranker. The first-stage retriever produces
the initial candidate pool, and all reranking methods are evaluated on
the same recalled candidates for fair comparison. Unless otherwise
specified, the candidate pool size is set to 50 and the VLM local view
size is set to 10. A VLM call denotes one listwise reranking call over
one local candidate view. All experiments are conducted without
chain-of-thought prompting or supervised fine-tuning.

Our method, denoted as \textbf{AMBER}, performs online Elo aggregation
to update candidate scores from local VLM listwise outputs. Unless
otherwise specified, AMBER uses view-conditioned pairwise entropy for
query-level budget allocation and head-aware observation-margin sampling
for candidate-level view selection. 

We evaluate AMBER under global budgets averaging 24 and 8 calls per query
in the main and low-budget settings, respectively.

\paragraph{Baselines.}
We compare AMBER with embedding-only retrieval baselines and several
multi-call VLM reranking baselines. The embedding-only baselines include
QQMM-Embed-v2~\cite{xue2025improvemultimodalembeddinglearning},
RzenEmbed-v2-7B~\cite{jian2025rzenembedcomprehensivemultimodalretrieval},
and Qwen3-VL-Embedding-8B~\cite{li2026qwen3vlembeddingqwen3vlrerankerunifiedframework}.
For VLM reranking, we compare with the following baselines:
\begin{itemize}[leftmargin=6pt]
    \item \textbf{AcuRank}~\cite{yoon2025acurankuncertaintyawareadaptivecomputation}:
    an uncertainty-aware adaptive reranking baseline.
    \item \textbf{Sliding Window}: applies listwise reranking over fixed local
    windows of the candidate list.
    \item \textbf{TourRank}~\cite{chen2025tourrankutilizinglargelanguage}:
    a tournament-style reranking baseline that compares candidates through
    grouped local competitions.
    \item \textbf{UniRank}: a non-adaptive uniform-budget baseline that assigns
    the same number of VLM calls to each query and repeatedly reranks the
    current top-10 candidates.
\end{itemize}

\subsection{Main Results}
\begin{table*}[t]
\centering
\small
\setlength{\tabcolsep}{4.2pt}
\renewcommand{\arraystretch}{0.98}
\vspace{-20pt}
\caption{
Main reranking results across datasets.
R@$K$ denotes Recall@$K$, and N@$K$ denotes NDCG@$K$.
QQMM-Embed-v2,
RzenEmbed-v2-7B,
and Qwen3-VL-Embedding-8B
are embedding-only retrieval baselines; all VLM reranking methods are
evaluated on the same top-50 candidate pools retrieved by
Qwen3-VL-Embedding-8B. For CIRCO, we report R@5 and mAP metrics, while
for PhotoBench, Avg. N@10 averages Chinese and English NDCG@10 scores.
All PhotoBench values are multiplied by 100. Avg. Rank is computed over
the displayed metrics within each dataset; lower is better. Best and
second-best results are shown in bold and underlined, respectively.
Full results are reported in Appendix~\ref{app:full_results}.
}
\vspace{-10pt}
\begin{tabular}{@{}llccccccc@{}}
\toprule
\multirow{9}{*}{\textbf{CIRR}}
& \textbf{Method} & \textbf{Calls}
& \textbf{R@1} & \textbf{R@2} & \textbf{R@5} & \textbf{R@10} & \textbf{R@50}
& \textbf{Avg. Rank} \\
\cmidrule(l){2-9}
& QQMM-Embed-v2            & --       & 38.43 & 59.80 & \textbf{80.36} & \textbf{89.23} & \textbf{97.59} & 3.40 \\
& RzenEmbed-v2-7B          & --       & 19.30 & 44.96 & 70.36 & 83.08 & 96.77 & 8.00 \\
& Qwen3-VL-Embedding-8B    & --       & 48.80 & 63.81 & 78.39 & 86.55 & \underline{96.94} & 4.00 \\
& AcuRank                  & $\sim$24 & 50.63 & 63.78 & 77.90 & 85.06 & \underline{96.94} & 5.00 \\
& Sliding Window           & 24       & 51.11 & 63.86 & 76.05 & 86.55 & \underline{96.94} & 4.00 \\
& TourRank                 & 24       & 51.37 & 64.63 & 77.13 & 85.52 & \underline{96.94} & 4.00 \\
& UniRank                  & 24       & \underline{53.42} & \underline{66.58} & \underline{80.19} & 86.55 & \underline{96.94} & \underline{2.20} \\
& \textbf{AMBER}         & 24       & \textbf{54.07} & \textbf{66.60} & 79.30 & \underline{87.64} & \underline{96.94} & \textbf{1.80} \\

\midrule
\multirow{9}{*}{\textbf{CIRCO}}
& \textbf{Method} & \textbf{Calls}
& \textbf{R@5} & \textbf{mAP@5} & \textbf{mAP@10} & \textbf{mAP@25} & \textbf{mAP@50}
& \textbf{Avg. Rank} \\
\cmidrule(l){2-9}
& QQMM-Embed-v2            & --       & 71.12 & 45.92 & 47.13 & 50.39 & 51.45 & 6.00 \\
& RzenEmbed-v2-7B          & --       & 60.75 & 32.20 & 34.19 & 37.41 & 38.61 & 8.00 \\
& Qwen3-VL-Embedding-8B    & --       & 63.38 & 44.60 & 45.96 & 48.96 & 50.08 & 7.00 \\
& AcuRank                  & $\sim$24 & 73.75 & 56.28 & \underline{56.88} & 58.49 & 59.52 & 2.80 \\
& Sliding Window           & 24       & 71.50 & 54.52 & 55.64 & 58.33 & 58.86 & 4.40 \\
& TourRank                 & 24       & \underline{74.00} & \underline{56.45} & 56.74 & \underline{59.30} & \underline{59.95} & \underline{2.20} \\
& UniRank                  & 24       & 71.62 & 55.09 & 53.09 & 56.00 & 57.13 & 4.60 \\
& \textbf{AMBER}         & 24       & \textbf{74.75} & \textbf{58.80} & \textbf{58.73} & \textbf{61.17} & \textbf{61.78} & \textbf{1.00} \\

\midrule
\multirow{8}{*}{\textbf{PhotoBench}}
& \textbf{Method} & \textbf{Calls}
& \textbf{CN N@1} & \textbf{CN N@10} & \textbf{EN N@1} & \textbf{EN N@10} & \textbf{Avg. N@10}
& \textbf{Avg. Rank} \\
\cmidrule(l){2-9}
& QQMM-Embed-v2            & --       & 46.69 & 49.65 & 44.38 & 47.47 & 48.56 & 6.00 \\
& RzenEmbed-v2-7B          & --       & 48.54 & 50.07 & 45.64 & 48.24 & 49.16 & 5.00 \\
& Qwen3-VL-Embedding-8B    & --       & 44.52 & 45.54 & 41.95 & 44.13 & 44.84 & 7.00 \\
& AcuRank                  & $\sim$24 & \underline{51.97} & 52.27 & 49.33 & 50.37 & 51.32 & 3.20 \\
& Sliding Window           & 24       & 51.38 & \underline{52.35} & 49.08 & 50.75 & \underline{51.55} & \underline{2.80} \\
& TourRank                 & 24       & 51.30 & 51.98 & \underline{50.50} & \underline{50.93} & 51.46 & 3.00 \\
& \textbf{AMBER}         & 24       & \textbf{54.06} & \textbf{53.89} & \textbf{52.60} & \textbf{52.27} & \textbf{53.08} & \textbf{1.00} \\
\bottomrule
\end{tabular}
\vspace{-10pt}
\label{tab:main_results}
\end{table*}

Table~\ref{tab:main_results} summarizes the results under the main
24-call setting across CIRR, CIRCO, and PhotoBench. AMBER achieves the strongest
overall performance among the compared reranking baselines under the
same top-50 candidate pools and comparable VLM-call budgets.

On CIRR, AMBER achieves the best R@1 and R@2 among all methods in the
table, and the strongest overall performance among VLM reranking baselines.
Compared with TourRank, the strongest reranking baseline on R@1, AMBER
improves R@1 from 51.37 to 54.07. This indicates that online Elo aggregation
is particularly effective for improving the very top position in
single-target composed image retrieval. Although QQMM-Embed-v2 obtains
higher R@5 and R@10 as an embedding-only reference, AMBER consistently
improves over the Qwen3-VL-Embedding-8B candidate pool used by all reranking
methods.

On CIRCO, AMBER achieves the best mAP across all reported cutoffs.
The gains are more pronounced on mAP@25 and mAP@50, indicating that
AMBER improves not only whether relevant images appear in the
candidate head, but also their relative ordering. This is important for
CIRCO, where each query may have multiple valid targets.

On PhotoBench, AMBER consistently improves ranking quality for both
Chinese and English queries. It obtains the best NDCG@1 and NDCG@10
for both languages, suggesting that the proposed budgeted multi-view
aggregation is also effective in personal photo retrieval scenarios.


\subsection{Low-Budget Results}

\begin{table*}[t]
\centering
\small
\setlength{\tabcolsep}{4.2pt}
\renewcommand{\arraystretch}{0.98}

\caption{
Low-budget reranking results on CIRR, CIRCO, and PhotoBench.
Fixed-budget methods use 8 VLM calls per query, while AcuRank reports
its average calls under native stopping. Best and second-best results
are shown in bold and underlined, respectively.
}
\label{tab:low_budget_results}
\vspace{-6pt}

\begin{tabular}{@{}llccccccc@{}}
\toprule

\multirow{7}{*}{\textbf{CIRR}}
& \textbf{Method}
& \textbf{Calls}
& \textbf{R@1}
& \textbf{R@2}
& \textbf{R@5}
& \textbf{R@10}
& \textbf{R@50}
& \textbf{Avg. Rank} \\
\cmidrule(l){2-9}

& Qwen3-VL-Embedding-8B
& --
& 48.80
& 63.81
& 78.39
& 86.55
& 96.94
& 3.20 \\

& AcuRank
& 10.53
& 49.95
& 62.89
& 75.23
& 80.60
& 96.94
& 4.40 \\

& Sliding Window
& 8
& 49.71
& 62.94
& 77.18
& \underline{88.63}
& 96.94
& 3.40 \\

& TourRank
& 8
& 51.57
& 63.69
& 74.36
& 81.74
& 96.94
& 3.80 \\

& UniRank
& 8
& \underline{53.23}
& \underline{66.68}
& \underline{79.49}
& 86.55
& 96.94
& \underline{2.00} \\

& \textbf{AMBER}
& 8
& \textbf{54.31}
& \textbf{67.35}
& \textbf{79.98}
& \textbf{88.72}
& 96.94
& \textbf{1.00} \\

\midrule

\multirow{7}{*}{\textbf{CIRCO}}
& \textbf{Method}
& \textbf{Calls}
& \textbf{R@5}
& \textbf{mAP@5}
& \textbf{mAP@10}
& \textbf{mAP@25}
& \textbf{mAP@50}
& \textbf{Avg. Rank} \\
\cmidrule(l){2-9}

& Qwen3-VL-Embedding-8B
& --
& 63.38
& 44.60
& 45.96
& 48.96
& 50.08
& 6.00 \\

& AcuRank
& 10.52
& 72.00
& 53.98
& 53.50
& 55.54
& 56.67
& 4.00 \\

& Sliding Window
& 8
& \underline{72.38}
& 54.32
& \underline{55.18}
& \underline{57.67}
& \underline{58.39}
& \underline{2.20} \\

& TourRank
& 8
& 67.12
& 53.06
& 53.18
& 56.43
& 57.11
& 4.00 \\

& UniRank
& 8
& 71.75
& \underline{54.82}
& 53.02
& 55.93
& 57.05
& 3.80 \\

& \textbf{AMBER}
& 8
& \textbf{73.12}
& \textbf{56.80}
& \textbf{56.76}
& \textbf{58.85}
& \textbf{59.70}
& \textbf{1.00} \\

\midrule

\multirow{7}{*}{\textbf{PhotoBench}}
& \textbf{Method}
& \textbf{Calls}
& \textbf{CN N@1}
& \textbf{CN N@10}
& \textbf{EN N@1}
& \textbf{EN N@10}
& \textbf{Avg. N@10}
& \textbf{Avg. Rank} \\
\cmidrule(l){2-9}

& Qwen3-VL-Embedding-8B
& --
& 44.52
& 45.54
& 41.95
& 44.13
& 44.84
& 6.00 \\

& AcuRank
& 10.69
& \underline{52.22}
& 50.57
& \underline{50.92}
& 49.30
& 49.93
& 3.00 \\

& Sliding Window
& 8
& 51.13
& \textbf{51.72}
& 49.08
& \underline{49.89}
& \underline{50.81}
& \underline{2.40} \\

& TourRank
& 8
& 49.62
& 50.49
& 49.92
& 49.73
& 50.11
& 3.40 \\

& UniRank
& 8
& 49.04
& 47.54
& 45.97
& 45.67
& 46.60
& 5.00 \\

& \textbf{AMBER}
& 8
& \textbf{52.72}
& \underline{51.56}
& \textbf{51.68}
& \textbf{50.38}
& \textbf{50.97}
& \textbf{1.20} \\

\bottomrule
\end{tabular}

\vspace{-8pt}
\end{table*}

Table~\ref{tab:low_budget_results} reports results under the reduced
budget setting. With 8 VLM calls per query, AMBER achieves the best
average rank on all three datasets, leading at R@1, R@2, R@5, and
R@10 on CIRR and all reported CIRCO metrics. On PhotoBench, it performs
best on four of the five displayed metrics, while Sliding Window is
slightly better on Chinese NDCG@10. AcuRank uses native stopping and
averages approximately 10.5--10.7 calls per query.
Additional experiments evaluate alternative aggregation rules, budget
scaling, practical inference costs, and generalization to an alternative
VLM reranker; see Appendix~\ref{app:additional_experiments}.

%% file: text/related_works.tex
\section{Related Works}

\subsection{Multimodal Image Retrieval}
Recent advancements leverage Multimodal Large Language Models (MLLMs) to create unified representation and ranking architectures \cite{lin2025mmembeduniversalmultimodalretrieval, li2026qwen3vlembeddingqwen3vlrerankerunifiedframework}, including training-free adaptations \cite{su2026retllmtrainingdatafreemllms}. Beyond similarity matching, methods have evolved to incorporate generative memorization \cite{li2024generativecrossmodalretrievalmemorizing}, conversational query rewriting \cite{hu2026recqrincorporatingconversationalquery}, and Chain-of-Thought reasoning \cite{wu2025chainofthoughtrerankingimageretrieval}. In the complex domain of Composed Image Retrieval (CIR), approaches range from tailored LLMs \cite{huynh2025collmlargelanguagemodel} and zero-shot visual-semantic integration \cite{ye2026finegrainedzeroshotcomposedimage} to autonomous cross-modal agents \cite{Yang_2026}. While these advancements significantly enhance single-query reasoning, scaling their heavy architectures across massive candidate pools remains a systemic bottleneck.

\subsection{LLMs for Listwise Ranking}
The viability of LLMs as listwise rerankers was established by proprietary models \cite{sun2024chatgptgoodsearchinvestigating, ma2023zeroshotlistwisedocumentreranking} and democratized by robust open-source alternatives like RankVicuna and RankZephyr \cite{pradeep2023rankvicunazeroshotlistwisedocument, pradeep2023rankzephyreffectiverobustzeroshot}. Research quickly shifted toward optimizing multi-call strategies, such as tournament-inspired schedules \cite{chen2025tourrankutilizinglargelanguage}, while analyzing inherent limitations like positional bias \cite{qiao2026llmbasedlistwisererankingeffect}. 

To mitigate severe computational costs, a wave of efficiency-driven innovations emerged, including ranking-aware passage compression \cite{compress-and-rank}, minimal reasoning frameworks \cite{liu2026stopoverthinkingunlockingefficient}, embedding-repurposed rerankers \cite{liu2025e2ranktextembeddingeffective}, and early-stopping strategies like EcoRank and AcuRank \cite{rashid2024ecorankbudgetconstrainedtextreranking, yoon2025acurankuncertaintyawareadaptivecomputation}. Recently, this efficient listwise paradigm has successfully expanded to handle visually rich, multimodal tasks \cite{cai2026visionmeetstextslistwise,sun2026efficientlistwisemultimodalreranking}.

Crucially, existing multi-call strategies either apply rigid computation uniformly (e.g., TourRank) or confine adaptivity to single queries (e.g., AcuRank, EcoRank), failing to dynamically redistribute budgets globally. AMBER overcomes this by converting local listwise outputs into continuous Elo updates to maintain a lightweight global ranking state. This enables a dual-level adaptive allocation that prioritizes informative query-candidate comparisons globally. Extended related works are detailed in the Appendix~\ref{app:related_works}.

%% file: text/conclusion.tex
\section{Conclusion}
In this paper, we present AMBER, an online, budgeted multi-view reranking framework that dynamically optimizes global resource allocation for listwise Vision-Language Models (VLMs). 

We further introduce a dual-level allocation strategy: dynamically constructing candidate views based on score ambiguity and scheduling queries to maximize expected information gain.

We theoretically connect AMBER's Elo updates to online
maximum-likelihood optimization under the Bradley--Terry model and
establish submodular approximation guarantees for the information-gain
formulation of query-level allocation.

Experiments on CIRR, CIRCO, and PhotoBench show that AMBER achieves
strong overall performance among the compared multi-call reranking
methods under comparable budgets and remains effective in low-budget
settings.

%% file: text/limitation.tex
\clearpage
\section{Limitations}

First, while our evaluation focuses on image retrieval, our modality-agnostic framework can naturally generalize to other domains like video or audio retrieval. Second, we employ standard listwise prompting without Chain-of-Thought (CoT) reasoning. Notably, AMBER consistently outperforms strong baselines even without CoT's added computational overhead, emphasizing the efficiency of our core allocation strategy. Finally, while our sensitivity analysis confirms robustness, the current hyperparameters are static; automatically tuning these parameters during the reranking loop remains a promising direction for future work.

%% file: text/Appendix.tex
\appendix
\section{Additional Experimental Results}
\label{app:additional_experiments}
\subsection{Full Experimental Results}
\label{app:full_results}

This appendix reports additional full-metric results for CIRCO and
PhotoBench. CIRR results are already fully reported in
Table~\ref{tab:main_results}.

\begin{table*}[h]
\centering
\small
\setlength{\tabcolsep}{4pt}
\caption{
Full results on CIRCO. Best results are shown in bold and
second-best results are underlined.
}
\resizebox{\textwidth}{!}{
\begin{tabular}{lccccccccc}
\toprule
\textbf{Method} & \textbf{Calls}
& \textbf{R@5} & \textbf{R@10} & \textbf{R@25} & \textbf{R@50}
& \textbf{mAP@5} & \textbf{mAP@10} & \textbf{mAP@25} & \textbf{mAP@50} \\
\midrule

QQMM-Embed-v2
& --
& 71.12 & 79.75 & 90.25 & \textbf{95.26}
& 45.92 & 47.13 & 50.39 & 51.45 \\

RZENEmbed-v2-7B
& --
& 60.75 & 72.38 & 85.75 & 92.25
& 32.20 & 34.19 & 37.41 & 38.61 \\

Qwen3-VL-Embedding-8B
& --
& 63.38 & 76.50 & 87.88 & 93.25
& 44.60 & 45.96 & 48.96 & 50.08 \\

AcuRank
& 23.28
& 73.75 & 83.38 & 87.38 & 93.25
& 56.28 & \underline{56.88} & 58.49 & 59.52 \\

Sliding Window
& 24
& 71.50 & \textbf{84.50} & \textbf{91.38} & 93.25
& 54.52 & 55.64 & 58.33 & 58.86 \\

TourRank
& 24
& \underline{74.00} & 82.38 & \underline{91.12} & 93.25
& \underline{56.45} & 56.74 & \underline{59.30} & \underline{59.95} \\

UniRank
& 24
& 71.62 & 76.50 & 87.88 & 93.25
& 55.09 & 53.09 & 56.00 & 57.13 \\

\textbf{AMBER}
& 24
& \textbf{74.75} & \underline{83.50} & 90.75 & 93.25
& \textbf{58.80} & \textbf{58.73} & \textbf{61.17} & \textbf{61.78} \\
\bottomrule
\end{tabular}
}

\label{tab:app_circo_results}
\end{table*}

\begin{table*}[h]
\centering
\small
\setlength{\tabcolsep}{3.5pt}
\caption{
Full PhotoBench results in the merged-album setting under our evaluation
protocol. Queries with empty ground-truth sets are retained and assigned
zero metric values instead of being skipped. All metric values are
multiplied by 100. Best results are shown in bold and second-best results
are underlined.
}
\resizebox{\textwidth}{!}{
\begin{tabular}{llcccccccc}
\toprule
\textbf{Lang.} & \textbf{Method} & \textbf{Calls}
& \textbf{R@1} & \textbf{R@5} & \textbf{R@10} & \textbf{R@20}
& \textbf{NDCG@1} & \textbf{NDCG@10} & \textbf{NDCG@20} \\
\midrule

CN & QQMM-Embed-v2
& --
& 23.56 & 44.53 & 51.63 & 58.51
& 46.69 & 49.65 & 51.62 \\

CN & RZENEmbed-v2-7B
& --
& 24.02 & 44.30 & 51.59 & 58.05
& 48.54 & 50.07 & 51.78 \\

CN & Qwen3-VL-Embedding-8B
& --
& 22.04 & 41.05 & 47.26 & 52.65
& 44.52 & 45.54 & 46.94 \\

CN & AcuRank
& 24.3
& 25.95 & 46.48 & 52.25 & 54.51
& 51.97 & 52.27 & 52.18 \\

CN & Sliding Window
& 24
& 25.85 & 46.14 & \underline{52.76} & 56.58
& 51.38 & 52.35 & \underline{52.89} \\

CN & TourRank
& 24
& 24.89 & 46.13 & 52.65 & \textbf{56.67}
& 51.30 & 51.98 & 52.60 \\

CN & \textbf{AMBER}
& 24
& \textbf{26.91} & \textbf{48.00} & \textbf{53.24} & \underline{56.62}
& \textbf{54.06} & \textbf{53.89} & \textbf{54.22} \\

\midrule

EN & QQMM-Embed-v2
& --
& 22.53 & 43.05 & 50.49 & 57.40
& 44.38 & 47.47 & 49.62 \\

EN & RZENEmbed-v2-7B
& --
& 22.86 & 43.42 & 50.89 & 57.68
& 45.64 & 48.24 & 50.20 \\

EN & Qwen3-VL-Embedding-8B
& --
& 21.20 & 40.64 & 46.57 & 52.32
& 41.95 & 44.13 & 45.63 \\

EN & AcuRank
& 24.2
& 24.62 & 45.45 & 51.00 & 54.21
& 49.33 & 50.37 & 50.68 \\

EN & Sliding Window
& 24
& 24.67 & 45.17 & 51.92 & \underline{56.61}
& 49.08 & 50.75 & 51.72 \\

EN & TourRank
& 24
& \underline{24.84} & \underline{45.81} & 51.64 & 55.75
& \underline{50.50} & \underline{50.93} & \underline{51.76} \\

EN & \textbf{AMBER}
& 24
& \textbf{25.94} & \textbf{46.64} & \textbf{52.26} & \textbf{56.72}
& \textbf{52.60} & \textbf{52.27} & \textbf{53.08} \\

\bottomrule
\end{tabular}
}
\label{tab:app_photobench_results}
\end{table*}

\subsection{Alternative Aggregation Methods}
\label{app:aggregation_methods}

\begin{table*}[t]
\centering
\small
\renewcommand{\arraystretch}{1.05}
\setlength{\tabcolsep}{10pt}
\caption{
Comparison of cross-view aggregation methods on CIRCO.
All methods use the same cached candidate views and VLM outputs,
with only the aggregation rule changed.
Best and second-best results are shown in bold and underlined,
respectively.
}
\label{tab:aggregation_methods}
\vspace{-4pt}

\begin{tabular}{lcccc}
\toprule
\textbf{Aggregation Method}
& \textbf{mAP@5}
& \textbf{mAP@10}
& \textbf{mAP@25}
& \textbf{mAP@50} \\
\midrule

Borda Count
& 54.12
& 54.84
& 57.54
& 58.11 \\

Average Rank
& 55.87
& 56.12
& 58.88
& 59.51 \\

Batch BT-MLE
& \underline{56.78}
& \underline{57.04}
& \underline{59.63}
& \underline{60.19} \\

\textbf{Elo}
& \textbf{58.80}
& \textbf{58.73}
& \textbf{61.17}
& \textbf{61.78} \\

\bottomrule
\end{tabular}

\vspace{-6pt}
\end{table*}

To isolate the effect of cross-view aggregation, we replay the same
cached candidate views and VLM outputs for all methods and change only
the aggregation rule. As shown in
Table~\ref{tab:aggregation_methods}, Elo achieves the best performance
at all four mAP cutoffs, outperforming batch BT-MLE, Borda Count, and
average-rank aggregation. These results support the use of Elo as an
effective online mechanism for aggregating fragmented listwise
rankings across views.

%

\subsection{Sensitivity Analysis}

\begin{table*}[t]
\centering
\scriptsize
\renewcommand{\arraystretch}{1.05}
\setlength{\tabcolsep}{2.5pt}
\caption{
Sensitivity analysis of the Elo initialization range on CIRCO under a
fixed budget of 8 VLM calls per query. The rating range controls the
strength of the embedding-prior initialization: the top-ranked
candidate from the first-stage retriever receives the highest initial
Elo score, while the last candidate receives the base score.
}
\vspace{-4pt}

\begin{tabular*}{\textwidth}{
@{\extracolsep{\fill}}
lcccccccc
@{}
}
\toprule
\textbf{Rating Range}
& \textbf{R@5}
& \textbf{R@10}
& \textbf{R@25}
& \textbf{R@50}
& \textbf{mAP@5}
& \textbf{mAP@10}
& \textbf{mAP@25}
& \textbf{mAP@50} \\
\midrule

100
& 73.00 & 82.25 & 89.38 & 93.25
& 56.36 & 56.39 & 58.52 & 59.42 \\

200
& \textbf{73.25} & \textbf{82.38} & 89.38 & 93.25
& \textbf{56.87} & \textbf{56.84}
& \textbf{58.88} & \textbf{59.76} \\

400
& 73.12 & 82.25 & \textbf{89.50} & 93.25
& 56.85 & 56.79 & 58.87 & 59.75 \\

800
& 73.12 & 82.25 & 89.38 & 93.25
& 56.69 & 56.74 & 58.79 & 59.68 \\

\bottomrule
\end{tabular*}

\label{tab:elo_init_sensitivity}
\vspace{-6pt}
\end{table*}

We further analyze the sensitivity of AMBER to the strength of the
embedding-prior Elo initialization. Table~\ref{tab:elo_init_sensitivity}
reports CIRCO results under a fixed budget of 8 VLM calls per query.
The performance remains stable across rating ranges from 100 to 800,
with the best overall results obtained around 200--400. This indicates
that AMBER does not rely on a finely tuned initialization scale: the
embedding prior provides a useful starting point, while online Elo
updates quickly incorporate VLM preference evidence.

\subsection{Budget-Performance Trade-off}
\label{app:budget_tradeoff}

\begin{table*}[t]
\centering
\scriptsize
\renewcommand{\arraystretch}{1.05}
\setlength{\tabcolsep}{2.5pt}
\caption{
Performance of AMBER on CIRCO under different VLM-call budgets.
Best and second-best results are shown in bold and underlined,
respectively. Ties shared by all budget settings are left unmarked.
}
\vspace{-4pt}

\begin{tabular*}{\textwidth}{
@{\extracolsep{\fill}}
lcccccccc
@{}
}
\toprule
\textbf{Calls / Query}
& \textbf{R@5}
& \textbf{R@10}
& \textbf{R@25}
& \textbf{R@50}
& \textbf{mAP@5}
& \textbf{mAP@10}
& \textbf{mAP@25}
& \textbf{mAP@50} \\
\midrule

4
& 71.88
& 81.00
& 88.00
& 93.25
& 55.65
& 55.42
& 57.33
& 58.43 \\

8
& 73.12
& 81.88
& 89.50
& 93.25
& 56.80
& 56.76
& 58.85
& 59.70 \\

12
& 73.38
& 82.62
& 90.25
& 93.25
& 57.50
& 57.54
& 59.72
& 60.50 \\

16
& \underline{73.75}
& \underline{83.00}
& \underline{90.62}
& 93.25
& \underline{58.02}
& \underline{58.05}
& \underline{60.38}
& \underline{61.07} \\

24
& \textbf{74.75}
& \textbf{83.50}
& \textbf{90.75}
& 93.25
& \textbf{58.80}
& \textbf{58.73}
& \textbf{61.17}
& \textbf{61.78} \\

\bottomrule
\end{tabular*}

\label{tab:budget_tradeoff}
\vspace{-6pt}
\end{table*}

Table~\ref{tab:budget_tradeoff} evaluates AMBER with budgets ranging
from 4 to 24 VLM calls per query. Performance improves consistently as
the available budget increases, with particularly clear gains on the
mAP metrics. At the same time, the 8-call setting already achieves
strong performance, demonstrating a favorable trade-off between
reranking quality and inference cost.


\subsection{Practical Efficiency}
\label{app:practical_efficiency}

\begin{table*}[t]
\centering
\scriptsize
\renewcommand{\arraystretch}{1.05}
\setlength{\tabcolsep}{3.5pt}
\caption{
Practical inference costs measured on the same CIRCO subset using a
shared API backend. The projected charge estimates the cost of
evaluating the full 800-query CIRCO test split.
}
\vspace{-4pt}

\begin{tabular*}{\textwidth}{
@{\extracolsep{\fill}}
lccccccc
@{}
}
\toprule
\textbf{Method}
& \shortstack{\textbf{Calls /}\\\textbf{Query}}
& \shortstack{\textbf{Images /}\\\textbf{Call}}
& \shortstack{\textbf{Input Tokens /}\\\textbf{Call}}
& \shortstack{\textbf{Output Tokens /}\\\textbf{Call}}
& \shortstack{\textbf{Latency / Call}\\\textbf{(s)}}
& \shortstack{\textbf{Cost / Call}\\\textbf{(CNY)}}
& \shortstack{\textbf{Projected Charge}\\\textbf{(CNY)}} \\
\midrule

AMBER
& 8
& 11.00
& 3327.6
& 39.2
& 1.899
& 0.001742
& 11.15 \\

UniRank
& 8
& 11.00
& 3361.6
& 38.0
& 1.763
& 0.001757
& 11.24 \\

TourRank
& 8
& 11.63
& 3419.1
& 41.4
& 1.957
& 0.001829
& 11.70 \\

Sliding Window
& 8
& 11.00
& 3301.6
& 38.0
& 1.825
& 0.001727
& 11.05 \\

AcuRank
& 10.52
& 10.12
& 3067.5
& 35.2
& 1.700
& 0.001604
& 13.50 \\

\bottomrule
\end{tabular*}

\label{tab:practical_efficiency}
\vspace{-6pt}
\end{table*}

Table~\ref{tab:practical_efficiency} shows that the compared methods
have similar per-call image counts, token usage, latency, and API
costs. This supports the number of VLM calls as a practical measure of
inference budget. AMBER additionally requires only lightweight
client-side updates to candidate ratings and observation counts.

\subsection{Results with an Alternative VLM Reranker}
\label{app:alternative_reranker}

\begin{table*}[t]
\centering
\scriptsize
\renewcommand{\arraystretch}{1.05}
\setlength{\tabcolsep}{2.5pt}
\caption{
Results on CIRCO using Qwen3.5-Flash as the VLM reranker.
Qwen3-VL-Embedding-8B remains the first-stage retriever, and all
methods rerank the same top-50 candidate pools.
Best and second-best results are shown in bold and underlined,
respectively. Ties shared by all methods are left unmarked.
}
\vspace{-4pt}

\begin{tabular*}{\textwidth}{
@{\extracolsep{\fill}}
lccccccccc
@{}
}
\toprule
\textbf{Method}
& \textbf{Calls}
& \textbf{R@5}
& \textbf{R@10}
& \textbf{R@25}
& \textbf{R@50}
& \textbf{mAP@5}
& \textbf{mAP@10}
& \textbf{mAP@25}
& \textbf{mAP@50} \\
\midrule

Sliding Window
& 8
& \underline{75.38}
& \textbf{84.75}
& \underline{90.50}
& 93.25
& \underline{59.75}
& \underline{59.32}
& \underline{61.61}
& \underline{62.29} \\

AcuRank
& 10.60
& 75.00
& 81.75
& 88.50
& 93.25
& 59.27
& 58.22
& 60.01
& 60.97 \\

UniRank
& 8
& 73.50
& 76.50
& 87.88
& 93.25
& 57.94
& 55.29
& 58.21
& 59.33 \\

TourRank
& 8
& 72.25
& 81.50
& \textbf{91.62}
& 93.25
& 58.26
& 58.28
& 61.03
& 61.56 \\

\textbf{AMBER}
& 8
& \textbf{75.75}
& \underline{84.12}
& 89.88
& 93.25
& \textbf{60.63}
& \textbf{60.01}
& \textbf{61.83}
& \textbf{62.68} \\

\bottomrule
\end{tabular*}

\label{tab:alternative_reranker}
\vspace{-6pt}
\end{table*}

Table~\ref{tab:alternative_reranker} evaluates the methods after
replacing the default VLM reranker with Qwen3.5-Flash while keeping
the first-stage retriever and candidate pools unchanged. AMBER achieves
the best R@5 and all four mAP metrics, while remaining competitive on
the other Recall cutoffs. These results indicate that AMBER's
effectiveness is not limited to the default VLM reranker.

\section{Extended Related Work}
\label{app:related_works}

\subsection{Multimodal Image Retrieval}

Recent advancements in multimodal image retrieval have heavily
leveraged Multimodal Large Language Models (MLLMs) to create unified
and versatile representations.
For instance, MM-Embed leverages MLLMs to establish a universal
framework capable of handling diverse multimodal retrieval tasks
through cohesive semantic alignment
\cite{lin2025mmembeduniversalmultimodalretrieval}.
Building on this push for unified architectures, the Qwen3-VL framework
provides a comprehensive state-of-the-art pipeline that harmonizes both
the multimodal embedding generation and the subsequent ranking stages
\cite{li2026qwen3vlembeddingqwen3vlrerankerunifiedframework}.
To further reduce deployment overhead, RETLLM proposes a training-free
and data-free approach, adapting existing MLLMs directly for multimodal
information retrieval without the need for extensive fine-tuning
\cite{su2026retllmtrainingdatafreemllms}.

Beyond traditional similarity matching, new paradigms are incorporating
generative strategies and sophisticated query interpretation. Generative
cross-modal retrieval methods have explored memorizing images directly
within the parameters of language models, fundamentally shifting how
visual retrieval and downstream generation tasks are executed
\cite{li2024generativecrossmodalretrievalmemorizing}. 
Beyond atomic image matching, recent work further explores the
composition of coherent image bundles by modeling relational relevance
among multiple images \cite{shan2026weaving}.
Furthermore, to
bridge the semantic gap between ambiguous user inputs and system
representations, ReCQR incorporates conversational query rewriting to
clarify user intent and substantially improve retrieval accuracy
\cite{hu2026recqrincorporatingconversationalquery}. To refine initial
retrieval results, advanced reasoning techniques have been integrated
into the concluding ranking phases. Notably, the application of
Chain-of-Thought (CoT) prompting has been shown to effectively guide
models through step-by-step logical evaluation, thereby improving the
accuracy of image retrieval re-ranking tasks
\cite{wu2025chainofthoughtrerankingimageretrieval}.

A particularly prominent sub-field driving recent innovation is
Composed Image Retrieval (CIR), where queries consist of both a
reference image and textual modification instructions. CoLLM addresses
this complex task by introducing a large language model specifically
tailored to process and synthesize these compound multimodal queries
\cite{huynh2025collmlargelanguagemodel}. Tackling CIR in
resource-constrained environments, recent work has integrated
complementary visual-semantic features to achieve fine-grained,
zero-shot retrieval without requiring paired training data
\cite{ye2026finegrainedzeroshotcomposedimage}. Adding an autonomous
dimension to this process, $X^R$ introduces intelligent cross-modal
agents that actively reason and take steps to resolve intricate
composed image retrieval queries \cite{Yang_2026}.
More recently, OSCAR further formulates agentic CIR as an
optimization-steered planning problem, using optimized retrieval
trajectories to guide the online planning process
\cite{wang2026oscar}.
\subsection{LLMs for Listwise Ranking}

Large Language Models (LLMs) have been increasingly explored in
recommendation and ranking systems, including applications to scoring
and ranking functions \cite{lin2025can}. Within this broader trend,
early work on listwise ranking investigated ChatGPT as a re-ranking
agent \cite{sun2024chatgptgoodsearchinvestigating}. Building on this
foundation, researchers demonstrated the viability of zero-shot listwise
document reranking using proprietary LLMs
\cite{ma2023zeroshotlistwisedocumentreranking}. This approach was
quickly democratized by the development of open-source alternatives
like RankVicuna, which proved that open-weight LLMs could also perform
effective zero-shot listwise reranking
\cite{pradeep2023rankvicunazeroshotlistwisedocument}. Subsequent
models, such as RankZephyr, further refined this paradigm to offer more
robust zero-shot listwise reranking capabilities
\cite{pradeep2023rankzephyreffectiverobustzeroshot}.

As LLM-based rerankers matured, research shifted toward optimizing
ranking strategies and addressing inherent model limitations. For
instance, TourRank introduced a tournament-inspired strategy to
evaluate and structure the document ranking process more effectively
\cite{chen2025tourrankutilizinglargelanguage}. Concurrently, other work
highlighted the persistent challenge of positional bias in LLMs,
thoroughly analyzing its effects on listwise reranking performance
\cite{qiao2026llmbasedlistwisererankingeffect}.

Efficiency and computational costs have also emerged as critical
bottlenecks in deploying listwise ranking systems. To mitigate these
issues, EcoRank proposed a budget-constrained approach for text
re-ranking that balances cost and performance
\cite{rashid2024ecorankbudgetconstrainedtextreranking}. Similarly,
AcuRank implemented an uncertainty-aware adaptive computation method to
dynamically allocate processing resources during reranking
\cite{yoon2025acurankuncertaintyawareadaptivecomputation}. Other
efficiency-driven innovations include the Compress-then-Rank framework,
which accelerates listwise reranking through ranking-aware passage
compression \cite{compress-and-rank}. Researchers have also explored
unlocking efficient reranking with minimal reasoning to prevent
computational overthinking
\cite{liu2026stopoverthinkingunlockingefficient}, while E2Rank
demonstrated that standard text embeddings can be repurposed as highly
efficient listwise rerankers
\cite{liu2025e2ranktextembeddingeffective}.

Most recently, the scope of listwise reranking has expanded beyond pure
text to incorporate multimodal information. Pioneering work has begun
exploring the intersection of vision and text within listwise reranking
frameworks to seamlessly handle diverse data types
\cite{cai2026visionmeetstextslistwise}. Building upon this, highly
efficient listwise multimodal reranking techniques have been developed
specifically to process and rank long, visually rich documents
\cite{sun2026efficientlistwisemultimodalreranking}.
Despite the growing emphasis on efficiency, existing multi-call and
budget-aware strategies share critical limitations. Fixed-schedule
traversal methods such as TourRank apply uniform computation across
queries, which can waste resources on easy queries or uninformative
comparisons. Conversely, adaptive methods such as AcuRank and EcoRank
mitigate this issue through early stopping, but their adaptivity is
primarily confined within individual queries and does not explicitly
redistribute a shared global budget across queries.

AMBER departs from both rigid schedules and isolated intra-query
heuristics. By converting local listwise outputs into continuous
Elo-style updates, AMBER maintains a lightweight cross-view ranking
state without an explicit preference matrix. Building on this state, it
couples candidate-level view construction with query-level budget
allocation, directing VLM calls toward uncertain comparisons that are
more likely to affect head-ranking quality.

More broadly, recent work has explored agentic information retrieval,
where LLM agents interact with information environments through
iterative reasoning and decision making \cite{zhang2024agentic}.
At the system level, practical LLM agents have also been studied
through externalized components such as memory, skills, protocols,
and harness engineering \cite{zhou2026externalization}.

\section{Theoretical Analysis}
\label{app:proofs}

We provide theoretical grounding for AMBER's three core mechanisms
under Assumption~\ref{ass:bt} (Bradley--Terry model) introduced in
Section~\ref{sec:theory}. All results below are stated under this
assumption.

\subsection{Elo Aggregation as Online Maximum Likelihood Estimation}

We first establish that AMBER's update rule corresponds exactly to
stochastic gradient ascent (SGA) on the BT log-likelihood.

\begin{proof}[Proof of Proposition~\ref{prop:mle}]
The BT log-likelihood for a single observation $c_{i}\succ c_{k}$ is
\begin{equation}
\begin{split}
  \ell(s_{ij},s_{kj})
  &=\log P(c_{i}\succ c_{k}\mid q_{j}) \\
  &=-\log\!\left(1+10^{(s_{kj}-s_{ij})/\tau}\right).
\end{split}
\end{equation}
Computing partial derivatives:
\begin{align}
  \frac{\partial\ell}{\partial s_{ij}}
    &=\frac{\ln 10}{\tau}\bigl(1-P(c_{i}\succ c_{k}\mid q_{j})\bigr),\\[4pt]
  \frac{\partial\ell}{\partial s_{kj}}
    &=-\frac{\ln 10}{\tau}\bigl(1-P(c_{i}\succ c_{k}\mid q_{j})\bigr).
\end{align}
The Elo updates in
Eqs.~\eqref{eq:elo_update_i}--\eqref{eq:elo_update_k} can then be
written as
\begin{equation}
\begin{split}
  s_{ij} \leftarrow{} & s_{ij}+K\bigl(1-P(c_{i}\succ c_{k}\mid q_{j})\bigr) \\
         ={} & s_{ij}+\frac{K\tau}{\ln 10}\cdot
            \frac{\partial\ell}{\partial s_{ij}},
\end{split}
\end{equation}

\begin{equation}
\begin{split}
  s_{kj} \leftarrow{} & s_{kj}-K\bigl(1-P(c_{i}\succ c_{k}\mid q_{j})\bigr) \\
         ={} & s_{kj}+\frac{K\tau}{\ln 10}\cdot
            \frac{\partial\ell}{\partial s_{kj}},
\end{split}
\end{equation}
which is precisely gradient ascent with step size
$\eta=K\tau/\ln 10$.
\end{proof}

\begin{remark}
The BT log-likelihood is globally concave but not strongly concave,
due to translation invariance (adding a constant to all scores leaves
every pairwise probability unchanged). A standard remedy is to add a
weak $\ell_{2}$ regularizer $-\lambda\|\mathbf{s}\|^{2}$, which breaks
this degeneracy without affecting the induced ranking. Under this
regularized objective, classical SGA theory
\citep{robbins1951stochastic} guarantees convergence to the unique MLE.
In AMBER, the embedding-prior initialisation
(Eq.~\eqref{eq:elo_init}) further anchors the score scale and plays a
role similar to such regularisation without requiring an explicit
penalty.
\end{remark}

\noindent
Corollary~\ref{cor:convergence} confirms that AMBER's Elo
aggregation provides a statistically consistent estimate of the true
candidate ordering, with error decreasing as more VLM calls are
processed.

\subsection{Sample Complexity and the Rationale for
              Head-Aware Observation-Margin Sampling}

Proposition~\ref{prop:mle} shows that Elo scores converge
asymptotically.  Under a finite budget $B$, however, not all candidate
pairs can be observed.  We now ask: \emph{which pairs must be observed,
and how many times, to guarantee correct top-$K$ recovery?}

\begin{proof}[Proof Sketch of Proposition~\ref{prop:sample}]
By Proposition~\ref{prop:mle}, each Elo score $\hat{s}_{ij}$ is an SGA
iterate on the BT log-likelihood.  By standard concentration results
for SGA on concave objectives
\citep{nemirovski2009robust,rakhlin2012making}, after $m$ observations
of pair $(c_{i},c_{k})$, the estimation error satisfies
$|\hat{s}_{ij}-s_{i}^{*}|\leq O(\tau/\sqrt{m})$ with high probability.
For the top-$K$ set to be correctly recovered, it suffices that for
every straddling pair the accumulated error is smaller than the gap:
\begin{equation}
  |\hat{s}_{ij}-s_{i}^{*}|+|\hat{s}_{kj}-s_{k}^{*}| < \Delta_{K}.
\end{equation}
Setting $O(\tau/\sqrt{m})<\Delta_{K}/2$ and applying a union bound
over all $O(N^{2})$ pairs yields Eq.~\eqref{eq:sample_complexity}.
\end{proof}

\noindent
Proposition~\ref{prop:sample} provides theoretical motivation for
the three factors in the candidate sampling value $a_{ij}$
(Eq.~\eqref{eq:sampling_value}). Specifically, head importance
focuses sampling near the top-$K$ boundary, the observation penalty
discourages redundant evaluations, and margin ambiguity prioritizes
candidates with uncertain local ordering. 
Together, these factors
translate the sample-complexity intuition into a practical
candidate-level sampling rule.

\subsection{Submodular Motivation for Query Scheduling}

We further provide an information-theoretic motivation for the
view-conditioned query scheduler. The practical scheduling score in
Eq.~\eqref{eq:scheduling_score} is a lightweight proxy for the marginal
value of one additional VLM call. To clarify this connection, we first
consider an idealized uncertainty-reduction objective.

\begin{definition}[Idealized uncertainty reduction]
For query $q_{j}$, let $\mathcal{A}_{q_{j}}$ denote the set of VLM
calls already made. Let $\mathbf{s}^{*}_{q_j}$ be the latent BT scores
for candidates in $\mathcal{C}_{q_j}$. Define the idealized information
gain as
\begin{equation}
  F_{q_j}(\mathcal{A}_{q_j})
  =
  H(\mathbf{s}^{*}_{q_j})
  -
  H(\mathbf{s}^{*}_{q_j}\mid \mathcal{A}_{q_j}).
  \label{eq:ideal_info_gain}
\end{equation}
\end{definition}

\begin{proposition}[Submodularity of idealized information gain]
\label{prop:submodular_app}
Under Assumption~\ref{ass:bt}, if VLM observations are conditionally
independent given $\mathbf{s}^{*}_{q_j}$, then
$F_{q_j}(\mathcal{A}_{q_j})$ is a non-negative, monotone, submodular
set function. That is, for any
$\mathcal{A}_{q_j}\subseteq\mathcal{B}_{q_j}$ and any additional call
$x\notin\mathcal{B}_{q_j}$,
\begin{equation}
\begin{split}
  & F_{q_j}(\mathcal{A}_{q_j}\cup\{x\})-F_{q_j}(\mathcal{A}_{q_j}) \\
  & \quad \geq
  F_{q_j}(\mathcal{B}_{q_j}\cup\{x\})-F_{q_j}(\mathcal{B}_{q_j}).
\end{split}
\end{equation}
\end{proposition}

\begin{proof}[Proof Sketch]
The marginal gain of adding a call $x$ is the conditional mutual
information
\begin{equation}
\begin{split}
  & I(\mathbf{s}^{*}_{q_j};x\mid \mathcal{A}_{q_j}) \\
  & \quad = H(\mathbf{s}^{*}_{q_j}\mid \mathcal{A}_{q_j})
   - H(\mathbf{s}^{*}_{q_j}\mid \mathcal{A}_{q_j}\cup\{x\}).
\end{split}
\end{equation}
Under conditional independence of observations given
$\mathbf{s}^{*}_{q_j}$, conditional mutual information exhibits
diminishing returns as the conditioning set grows. Therefore, the
information-gain objective is monotone and submodular.
\end{proof}

\noindent
For the idealized objective
$\sum_j F_{q_j}(\mathcal{A}_{q_j})$, the classical result of
\citet{nemhauser1978analysis} implies that greedy maximisation under a
cardinality budget achieves a $(1-1/e)$ approximation to the optimal
allocation. Directly computing the marginal information gain, however,
requires maintaining a posterior distribution over latent BT scores,
which is impractical in our online VLM reranking setting.

AMBER therefore uses the view-conditioned scheduling score
$U_{q_j}$ in Eq.~\eqref{eq:scheduling_score} as a tractable proxy for
this marginal gain. Specifically, $H_{\mathrm{pair}}$ estimates the
uncertainty of the exact next view that would be reranked, while
$w_{\mathrm{head}}$ biases the allocation toward comparisons that are
more consequential for head-ranking metrics.

The function $f_{q_j}$ in
Definition~\ref{def:uncertainty_reduction} can be viewed as a
view-restricted, head-weighted surrogate for the idealized
information-gain objective $F_{q_j}$.

Thus, the submodular analysis motivates the greedy scheduling design,
while the empirical results validate the effectiveness of the
practical proxy.

\paragraph{Limitations of the theoretical analysis.}
The results above rest on Assumption~\ref{ass:bt}.  In practice, VLM
outputs may exhibit positional bias and context sensitivity, meaning the
true pairwise distributions deviate from the BT model.  Furthermore,
Proposition~\ref{prop:submodular_app} assumes conditional independence
across calls, whereas AMBER's Elo state explicitly couples successive
observations.  These gaps motivate the empirical evaluation in
Section~\ref{sec:experiments}, which provides complementary evidence
for AMBER's effectiveness under realistic conditions.